\documentclass[twoside]{article}

\usepackage{preamble}
\usepackage{notation}
\usepackage[accepted]{aistats2025}
\usepackage{dsfont}
\usepackage{natbib}

\begin{document}

\runningauthor{A. Scheid, E. Boursier, A. Durmus, M.I. Jordan, P. Ménard, E. Moulines, M. Valko}

\twocolumn[

\aistatstitle{Optimal Design for Reward Modeling in RLHF}

\aistatsauthor{Antoine Scheid \And Etienne Boursier \And Alain Durmus}

\aistatsaddress{  CMAP - CNRS \\
\'Ecole polytechnique \\
Palaiseau, France
\And  INRIA Saclay
\\ Université Paris Saclay, LMO \\
Orsay, France
\And  CMAP - CNRS \\
\'Ecole polytechnique \\
Palaiseau, France} 

\aistatsauthor{Michael I. Jordan \And Pierre Ménard \And Eric Moulines \And Michal Valko}

\aistatsaddress{U.C., Berkeley \\
INRIA, ENS
\\
Paris, France
\And 
ENS Lyon
\\
Lyon, France
\And 
CMAP - CNRS \\
\'Ecole polytechnique
\\
Palaiseau, France
\And
INRIA}]

\begin{abstract}
     Reinforcement Learning from Human Feedback (RLHF) has become a popular approach to align language models (LMs) with human preferences. This method involves collecting a large dataset of human pairwise preferences across various text generations and using it to infer (implicitly or explicitly) a reward model. Numerous methods have been proposed to learn the reward model and align a LM with it. However, the costly process of collecting human preferences has received little attention and could benefit from theoretical insights. This paper addresses this issue and aims to formalize the reward training model in RLHF. We frame the selection of an effective dataset as a simple regret minimization task, using a linear contextual dueling bandit method. Given the potentially large number of arms, this approach is more coherent than the best-arm identification setting. We then propose an offline framework for solving this problem. Under appropriate assumptions — linearity of the reward model in the embedding space, and boundedness of the reward parameter — we derive bounds on the simple regret. Finally, we provide a lower bound that matches our upper bound up to constant and logarithmic terms. To our knowledge, this is the first theoretical contribution in this area to provide an offline approach as well as worst-case guarantees.
\end{abstract}

\section{Introduction}


In learning from human feedback \citep{christiano2017deep, naveed2023comprehensive, wei2023overview}, an agent learns to act based on a preference signal. This subject has recently seen a surge of interest due to its effectiveness for aligning pre-trained Large Language Models (LLMs) with human preferences. Typically, the human feedback is gathered by constructing a large dataset of contexts (prompts), pairs of language model outputs (completions), and human preference between the pairs. Given this preference dataset, several methods have been proposed to align pre-trained large language models (LLMs). For instance, reinforcement learning from human feedback (\texttt{RLHF}, \citealt{ziegler2019finetuning}) involves training a reward model from the preference dataset and then fine-tuning the pre-trained LLM using reinforcement learning, typically with the \texttt{PPO} algorithm \citep{schulman2017proximal}, to maximize the reward model. Another approach is the direct preference optimization (\texttt{DPO}) procedure \citep{rafailov2024direct}, where the pre-trained LLM is directly fine-tuned with the preference dataset by learning an implicit reward model.

Learning from human feedback has proven to have extraordinary abilities in its application to various fields, including robotics \citep{tucker2020preference, biyik2020active, biyik2024active}, language models \citep{ouyang2022training, touvron2023llama, bubeck2023sparks}, and recommendations \citep{chen2024softmax, zhao2023recommender}. However, the majority of works in this area focus on preference optimization, and little is known about how to efficiently construct a human preferences dataset. In this work, we provide a theoretically grounded insight into the data collection process for learning a reward model.

In general, a very large dataset of prompts and associated generations $\dini$ is sampled. Among this dataset, a smaller one is selected ($\dselect$) and receives feedback from human labelers, due to the cost of labeling the generations. In practice, the dataset selection is achieved without too much care, hence a loss in the information that could have been retrieved from the original dataset $\dini$. Some techniques to improve the selection rely on heuristics or black-box methods \citep{shen2024towards, dong2024rlhf, chang2024dataset} but lack of provable bounds on the optimality of the procedure.

Based on these considerations, we study the offline selection of the optimal dataset $\dselect$. Our goal is to minimize the number of samples that need to be rated by labelers while retaining as much valuable information as possible from the initial dataset $\dini$.
To achieve this, we propose a new method called $\algodpo$: \texttt{Optimal Design for Policy Optimization}. This method guides the dataset selection process using the solution to an optimal design problem. We prove that $\algodpo$ is optimal from a worst-case perspective. Interestingly, $\algodpo$ can be applied to select pairs in the dataset $\dselect$ at a low cost before running any preference optimization procedure.

We summarize our contributions as follows:
\begin{itemize}
    \item Under the Bradley-Terry model and the assumption of a contextual linear bandit for the reward model, we formalize a pure exploration bandit framework for the collection of samples used to train the reward model in RLHF.
    \item Within that setting, we introduce $\algodpo$: \texttt{Optimal Design for Preference Optimization}, which optimally selects the best arms to learn the reward model and we upper bound its simple regret.
    \item We prove the optimality of our technique with a lower bound which matches our upper bound up to logarithmic factors.
\end{itemize}

Note that few works study the optimal way of choosing the dataset of human preferences. Thus, we are very enthusiastic about the potential impact of our method and theoretical results on the reward training steps. As mentioned by \citet{casper2023open}, collecting data that is representative of human preferences is an open problem in RLHF and deserves more attention, hence our attempt in this direction.

\section{Setting}

\subsection{Background on RLHF}

For what follows, $\cX$ represents the set of contexts (or prompts) and $\cY$ the set of generations (or completions). Human labelers are presented with pairs of prompt-completion tuples, denoted $(x, y)$ and $(x, y')$, which we can express as $\{x, y, y'\}$. Formally, a language model $\phi$ is a mapping from the set of contexts $\cX$ to probability distributions over the set of possible generations $\cY$. The task of the labelers (annotators) is to determine which completion between $y$ and $y'$ is more accurate or preferable in the context of $x$, denoted as ${y \succ y' | x}$, when $(x, y)$ is preferred over $(x, y')$. To account for human uncertainty, we model the binary feedback $\1({y \succ y' | x})$ process probabilistically by assuming a preference probability $\P$: the event $\{y \succ y' | x\}$ coded as a binary variable $\1({y \succ y' | x})$ occurs with probability $\P({y \succ y' | x})$. 

The Bradley-Terry model \citep{bradley1952rank} provides a framework for modeling preferences based on real-valued rewards. Given a reward function $r(x, y)$ that assigns a score to each context-generation pair $(x, y)$, the probability of favoring one generation over another is expressed as follows
\begin{align}\label{equation:bradley_terry}
    \P(y \succ y'|x) & = \sigma(r(x,y)-r(x,y')) \\
    & = \nonumber 1/(1+e^{-(r(x,y)-r(x,y'))}) \eqsp,
\end{align}
where $\sigma$ is the sigmoid function. It is worth noting that alternative preference models, such as the Plackett-Luce model \citep{plackett1975analysis}, can be used instead of this one. After having selected a dataset of pairs $\dselect = \{X_t, Y_t^{(1)}, Y_t^{(2)}\}_{t \in [T]}$ and the associated human preferences to make it $\dselected = \{X_t, Y_t^{(1)}, Y_t^{(2)}, \1(Y_t^{(1)}\succ Y_t^{(2)}|X_t)\}_{t \in [T]}$, the estimated reward function $\hat{r}$ is computed as the minimum of the loss
\begin{equation}\label{equation:loss_rlhf}
\cJ(r) = -\E_{(x,y,y')\sim \dselect}[\log(\P(y\succ y'|x))] \eqsp,
\end{equation}
and is then used to fine-tune the model $\phi$ which needs to maximize this reward while staying close from the initial model $\phi_0$, which is achieved by minimizing the following loss
\begin{equation}\label{equation:loss_policy_rlhf}
    \cL(\phi) = \E_{\phi}[\hat{r}(x,y)] - \gamma \kldiv(\phi||\phi_0) \eqsp,
\end{equation}
where $\gamma$ is some constant and $\kldiv$ stands for the Kullback-Leibler divergence.

Despite a growing literature around the optimization procedures \eqref{equation:loss_rlhf}, \eqref{equation:loss_policy_rlhf} \citep[see, e.g.,][]{schulman2017proximal, rafailov2024direct, azar2024general}, little has been done to select optimally the human-labelled dataset $\dselect$, although it crucially impacts the reward training or the policy optimization.

In practice, the selection of the pairs $(y_n^1, \ldots, y_n^K)$ associated with the $n$-th prompt $x_n$ is achieved with the initial model $\phi_0$ sampling several generations for the same context (with a change of the temperature between the different ones) from which two generations are randomly selected. The full dataset is then given to labelers for rating. $T$ and $N$ are of order $1000$ to $100 000$ for usual datasets while $K$ is around one or a few dozens.

\subsection{RLHF as a Dueling Bandit Problem}

We now introduce our offline setting, which is one of the main novelty of our approach as compared to previous works around this topic. Relying on optimal design and statistical foundations, the objective is to choose a dataset $\dselect$ of prompts-generations that maximizes the information gained from human labelers' feedback. Our approach to minimize the size of the collected dataset is notable for two reasons. First, it aligns with common practice, as it is impractical to operate online and get labelers' feedback before choosing the next pair. Second, we establish matching upper and lower bounds, up to constant and logarithmic factors, which ensures the efficiency of $\algodpo$.

We make the assumption of a contextual linear reward, hence the existence of a known feature map $\psi \colon \cX \times \cY \to \R^d, x, y \mapsto \psi(x,y)$ such that for any $x,y \in \cX \times \cY$
\begin{equation}\label{equation:definition_linear_reward}
    r(x,y) = \langle \theta^\star, \psi(x,y) \rangle \eqsp.
\end{equation}
The reward is given with respect to the embedding of the promt-completion pair. Simple encoder models such as \texttt{BERT}, \texttt{RoBERTa} or \texttt{SBERT} \citep{reimers2019sentence,devlin2018bert,liu2019roberta} can be used for the embedding step. Usually, the feature map can be obtained by removing the last layer of the initially trained model.
\\

\begin{figure}[!htb]\centering
\includegraphics[width=0.47\textwidth,trim=1cm 0cm 1cm 1.5cm]{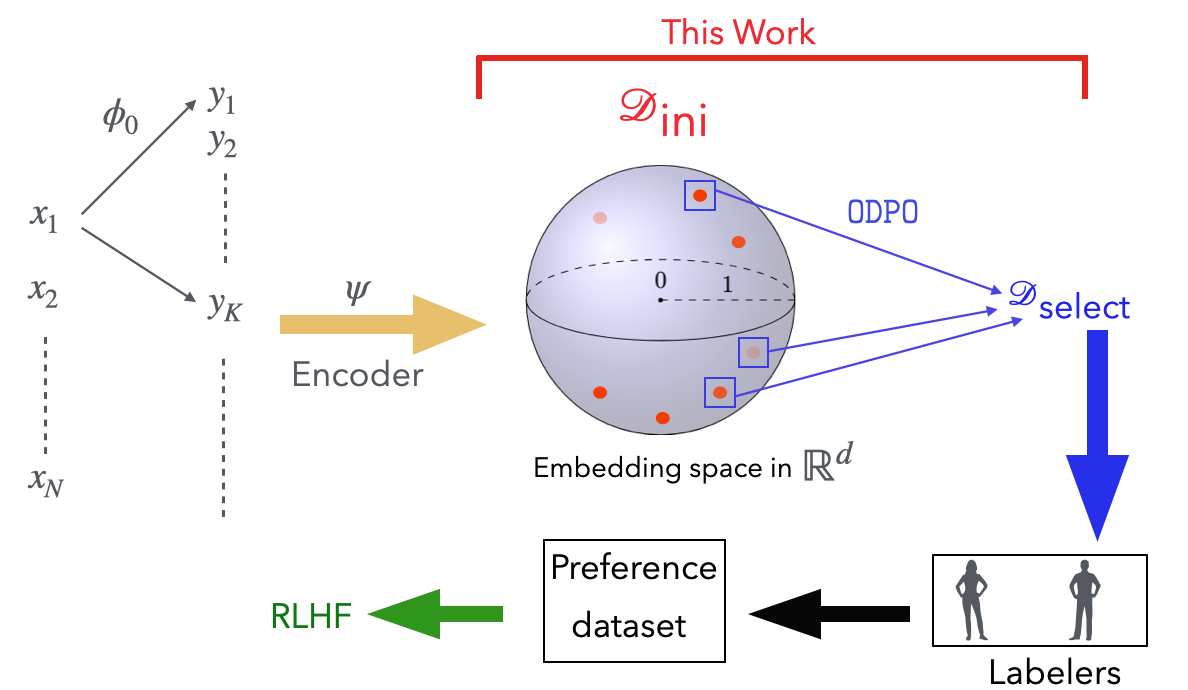}
    \caption{Illustration of ODPO among the whole RLHF framework.}
\label{figure:regret_comparison_non_contextual}
\end{figure}

We now work in the embedding space $\R^d$ for the dataset selection. Defining $N$ as the number of available prompts, and $K$ the number of associated generations for each prompt (although our results still hold for $N, K$ tending to infinity), we have access to an initial dataset that can be written $\{(x_n, y_n^k)\}_{n \in [N], k \in [K]}$. We model it as a union of $N$ sets $\cA_n, n \in [N]$ and write $\dini = \cup_{n \in [N]}\cA_n$. A subset $\cA_n$ represents the set of all the prompts associated with the same generation $x_n$: for any $n \in [N], \cA_n = \{a_n^1, \ldots, a_n^K\}$ and $a_n^k = \psi(x_n, y_n^k)$ with $n \in [N], k \in [K]$. The goal of our procedure is to select a subset of $T \in \N^\star$ pairs of generations, each pair being associated with the same context. Formally, it boils down to choosing a sequence $\{(\al_t, \am_t)\}_{t \in [T]}$ of pairs in $\dini$ that receive a feedback from labelers such that for any $t \in [T], (\al_t, \am_t) = (a_n^i, a_n^j)$ for some $n \in [N], i, j \in [K]$. The selected dataset $\dselect = \{(\al_t, \am_t)\}_{t \in [T]}$ is given to labelers who give the winner of the duel between $\al_t$ and $\am_t$ for any $t \in [T]$. The preference feedback is then received as
\begin{align*}
    & \{\al_t \succ \am_t\} \; \text{ with probability } \sigma(\langle \theta^\star, \al_t - \am_t \rangle) \; \\
    \text{ or } & \{\am_t \succ \al_t\} \; \text{ with probability } \sigma(\langle \theta^\star, \am_t - \al_t \rangle) \eqsp,
\end{align*}
and we encode it as a random variable $Y_t$, following
\begin{equation}\label{equation:definition_Y}
    Y_t = \1(\al_t \succ \am_t) \eqsp,
\end{equation}
which therefore follows the distribution
\begin{align}\label{equation:probability_Y}
    \P_{\theta^\star}(Y_t=1|\cF_{t-1}) & = \sigma(\langle \theta^\star, \al_t-\am_t\rangle) \\
    & \nonumber = 1/(1+e^{-\langle \theta^\star, \al_t-\am_t\rangle}) \eqsp,
\end{align}
for the unknown reward parameter $\theta^\star$. We also define the set of differences between all possible action pairs $\cB = \{a_n^i - a_n^j\}_{n \in [N], i, j \in [K]} = \cup_{n \in [N]} \{\cA_n - \cA_n\}$, as well as its cardinal $L = \card(\cB) \leq T K^2$. Choosing an arbitrary ordering among the elements of $\cB$, we can write $\cB = \{b_l\}_{l \in [L]}$.

It is important to note that we considered a finite action space for the sake of clarity. However, \textit{our theory still holds for an arbitrary action space of possibly infinite size}. We would keep the same bounds depending solely on $T$ and $d$ due to the leverage of optimal design theory which circumvents the burden of having a large action space and relies on a distribution $\pi^\star$ with a finite and bounded support.

Based on the feedback $(Y_t)_{t \in [T]}$ from the preference pairs, our procedure first estimates $\theta^\star$. Then, for any context $x_n$ given as an argument, a completion $y_n^i, i \in [K]$ or equivalently an action $\hat{a}_{T}(\cA_n) \in \cA_n$ can be chosen. We now define the \textit{simple regret} of an algorithm $\alg$, as
\begin{equation}\label{equation:definition_regret}
    \regret_{\alg}(T, (\cA_n)_{n \in [N]}, \theta^\star) = \max_{n \in [N]}\max_{a \in \cA_n} \ps{\theta^\star}{a- \hat{a}_T(\cA_n)} \eqsp,
\end{equation}
and defining $a_n^\star = \argmax_{a \in \cA_n}\langle \theta^\star, a\rangle$, we have that
\begin{equation*}
    \regret_{\alg}(T,(\cA_n)_{n \in [N]}, \theta^\star) = \max_{n \in [N]} \ps{\theta^\star}{a_n^\star - \hat{a}_T(\cA_n)} \eqsp,
\end{equation*}
Note that in our setup, we are looking for an algorithm that converges for \textit{any} possible parameter $\theta^\star$ in the unit ball. We also consider \textit{any} set of actions $(\cA_n)_{n \in [N]}$ - which makes our results robust in an adversarial, non i.i.d. setting. \textit{Minimizing the simple regret} is a coherent objective to study optimal dataset selection. Since we only care about selecting \textit{informative} arms and about the quality of the prediction after all the arms have been sampled, it makes more sense than looking at the cumulative reward. Imagine that labelers need to rate a pair of \textit{bad} completions: there is no harm for anyone. Secondly, it is hard to make hypotheses on the form of the actions sets or the reward parameters for embeddings of LM generations, hence the fact that we do not make any i.i.d. assumption.

\textit{Objectives.} Note that we could have thought about different objectives for our problem, instead of minimizing the simple regret:
\begin{itemize}
    \item \textit{Best-arm identification}: one formulation of it within our setup would be to maximize $\P(\hat{a}_T(\cA_n)=a^\star_n)$ over any $n \in [N]$.
    \item \textit{Arm distance minimization}: for any $T \in \N, n \in [N]$, minimize $\|\hat{a}_T(\cA_n)-a^\star_n\|$ for some well-chosen norm $\| \cdot \|$. 
\end{itemize}
We do not focus on \textit{arm distance minimization} here, as it is difficult to quantify the difference in quality between two LM generations based on their distance or cosine similarity in $\R^d$ \citep{steck2024cosine}. Additionally, since the reward gap between two arms can be arbitrarily small and approach zero, the concept of \textit{best-arm} does not really apply to our setup. Instead, seeking simple regret minimization effectively captures the quality of the dataset selection process by measuring how well the sampled pairs align the reward model with human preferences - with the one-step final reward becoming close from optimality.

We make the following assumption about the reward parameter as well as the embeddings of the pairs in $\R^d$.

\begin{assumption}[Boundedness of action and parameter]\label{assumption:boundedness}
    For any $x,y \in \cX \times \cY, \psi(x,y) \in \oball(0,1)$, where $\oball(0,1)$ stands for the unit ball in $\R^d$. We also suppose that $\theta^\star \in \oball(0,1)$.
\end{assumption}

\subsection{Parameter Estimation}

\textit{Log-likelihood and design matrix.} Several useful quantities appear in the rest of the paper. Since they are at the core of our algorithms and results, we introduce them now. We define the regularized log-likelihood $\cL$ for the collected samples up to time $t$ and a reward parameter $\theta \in \R^d$ as
\begin{align}\nonumber
    \cL_t(\{\al_s, \am_s, Y_s\}_{s\in [t-1]}, \theta) & = \sum_{s=1}^{t-1} \log( \P_\theta(\al_s, \am_s, Y_s)) \\
    & \label{equation:definition_likelihod} \quad - \lambda \|\theta\|_2^2 /2 \eqsp,
\end{align}
which can be rewritten as
\begin{align*}
    & \cL_t(\{\al_s, \am_s, Y_s\}_{s \in [t-1]}, \theta) = \sum_{s=1}^{t-1} Y_s \log(\sigma(\langle \theta, \al_s-\am_s\rangle)) \\
    & \quad + \sum_{s=1}^{t-1}(1- Y_s)\log(\sigma(-\langle \theta, \al_s-\am_s\rangle)) - \lambda \|\theta\|_2^2 /2\eqsp,
\end{align*}
and the maximum likelihood estimator at step $t$ (MLE) $\hat{\theta}_t$ is computed following
\begin{equation}\label{equation:definition_mle}
\hat{\theta}_t \in \argmax_{\theta \in \R^d} \cL_t(\{\al_s, \am_s, Y_s\}_{s \in [t-1]}, \theta) \eqsp.
\end{equation}

\begin{restatable}{lemma}{differentiationequationlikelihood}\label{lemma:differentiation_equation_likelihood}
We can differentiate the likelihood defined in \eqref{equation:definition_likelihod}, and obtain
\begin{align}\label{equation:differentiation_likelihood}
    & \nabla_\theta \cL(\{\al_s, \am_s, Y_s\}_{s \in [t-1]}, \theta) \\
    & \quad \nonumber =  \sum_{s=1}^{t-1} (Y_s - \P_{\theta}(Y_s=1))(\al_s - \am_s) - \lambda \theta \eqsp,
\end{align}
which gives by definition of the maximum likelihood estimator in \eqref{equation:definition_mle}, that $\hat{\theta}_t$ must satisfy
\begin{equation}\label{equation:solution_hat_theta}
    \sum_{s=1}^{t-1} (Y_s - \P_{\hat{\theta}_t}(Y_s=1))(\al_s - \am_s) - \lambda \hat{\theta}_t = 0 \eqsp.
\end{equation}
\end{restatable}

Defining for any $t \in [T+1]$ the function $H_t \colon \R^d \to \R^d$, $\theta \mapsto \lambda \theta + \sum_{s=1}^{t-1} \sigma(\ps{\theta}{\al_s-\am_s})(\al_s-\am_s)$, we obtain by definition of $\thmle$ that
\begin{equation}\label{equation:H_for_hattheta}
    \sum_{s=1}^{t-1} Y_s(\al_s-\am_s) = H_t(\thmle) \eqsp.
\end{equation}
Note that we define the MLE as the maximizer of the likelihood over the whole set $\R^d$ although we know that under \Cref{assumption:boundedness}, $\theta^\star \in \oball(0,1)$. This is why we define the projected MLE estimator $\pmle$ - a similar quantity is also used  by \citet{faury2020improved} - as
\begin{equation}\label{equation:definition_projected_mle}
    \pmle = \argmin_{\theta \in \oball(0,1)} \|H_t(\theta) - H_{t}(\hat{\theta}_{t})\|_{V_{t}^{-1}} \eqsp,
\end{equation}
where $\desmat_t$ stands for the design matrix in our problem, defined as
\begin{equation}\label{equation:design_matrix}
    \desmat_t = \lambda I + \sum_{s=1}^{t-1} (\al_s - \am_s)(\al_s - \am_s)^T \eqsp,
\end{equation}
where $\lambda >0$ is a regularization parameter - the same as for the likelihood. Since we work in an offline fashion, we have access to all the data $\{a_n^k\}_{n \in [N], k \in [K]}$ and work on it, trying to extract as much knowledge as possible from the pairwise comparisons.

Note that a lot of applied works help to circumvent the burden of manipulating very large set of parameters in language modeling \citep{hu2021lora, houlsby2019parameter, lester2021power}, such as the computation of inversion of matrices.

\section{Offline and Online Algorithms}\label{section:algorithm}

Algorithms for our setting sample $T$ pairs from the set $\dini$ (possibly with repetition). For any $t \in [T]$, we write $(\al_t, \am_t)$ for the pair of actions sampled by $\algodpo$ at iteration $t$, even though there is no "time ordering" in the procedure. This choice corresponds to taking an action $B_t \in \cB \colon B_t = \al_t - \am_t$. \eqref{equation:probability_Y} shows that for a reward parameter $\theta$, we have $Y_t \sim \ber(\sigma(\langle \theta, \al_t-\am_t\rangle)) = \ber(\sigma(\theta^T B_t))$. 

\subsection{$\algodpo$: Optimal Design Policy}

We now introduce $\algodpo$: \texttt{Optimally Designed Policy Optimization}, with the pseudocode provided in \Cref{algorithm:odpo_policy}. The core idea behind is to work offline using the entire dataset $\dini$. The strength of optimal design techniques comes from the \textit{Kiefer-Wolfowitz theorem} \Cref{appendix:algorithms})
to select an optimal core subset of samples. This makes it ideal for selecting $\dselect$ without requiring any online feedback. In our approach, the approximate optimal design policy $\hat{\pi}$ is obtained using the \textit{Frank-Wolfe algorithm}, given in \Cref{appendix:algorithms}
. Instead of requiring $T$ steps of computations of the likelihood and most informative pairs, our setup only requires to run the \textit{Frank-Wolfe algorithm} to know which subset $\dselect$ to select from $\dini$. Then, after the human preferences over $\dselect$ are given, the likelihood and the MLE are only computed once.

After sampling $T$ informative pairs $\{\al_t, \am_t\}_{t \in [T]}$ from $\dini$, $\algodpo$ constructs the maximum likelihood estimator $\hat{\theta}_T$ for the regularized log-likelihood relying on $\dselected = \{\al_t, \am_t, Y_t\}_{t \in [T]}$ and the maximum likelihood estimator $\hat{\theta}_{T+1}^P$ projected on the unit ball - see \eqref{equation:definition_projected_mle}. Then, for any context $x_n, n \in [N]$ associated with the embedded set $\cA_n$, and the estimated reward parameter $\hat{\theta}_{T+1}^P$ output by $\algodpo$, we can estimate the best-arm in this set $\hat{a}_T(\cA_n)$, following an optimistic procedure
\begin{equation}\label{equation:definition_estimated_action}
    \hat{a}_T(\cA_n) = \argmax_{a \in \cA_n} \langle\hat{\theta}^P_{T+1}, a\rangle \eqsp.
\end{equation}

\begin{algorithm}[!ht]
\caption{$\algodpo$: \texttt{Optimally Designed Policy Optimization}}\label{algorithm:odpo_policy}
\begin{algorithmic}[1]
    \State {\bfseries Input:} Number of samples $T$, set of actions $\cB = \cup_{n \in [N]} \{a_n^i-a_n^j\}_{n \in [N], i, j \in [K]}$ of size $L$, regularization parameters $\lambda$, approximation parameter, $\epsilon$.
    \State {\bfseries Compute} the history $\cH_{0} = \varnothing$, as well as $t=0, \hat{\theta}_0 = \varnothing$ and $V_0 = \lambda I$, $\dselect = \varnothing$.
    \State Use $\algfrankwolfe$ to compute an $(1+\epsilon)$ approximation $\hat{\pi}$ of the optimal design $\pi^\star$ with $\cB, \cU(\cB), \lambda$.
    \For{$b \in \cB$}
        \State Append $b$ to $\dselect$ $\lceil T\hat{\pi}_b \rceil$ times.
    \EndFor
    \State Compute $\hat{\theta}_{T+1}$ according to \eqref{equation:definition_mle} and $\hat{\theta}_{T+1}^P$ according to \eqref{equation:definition_projected_mle}.
    \State Return $\dselect$ and $\hat{\theta}_{T+1}^P$.
\end{algorithmic}
\end{algorithm}

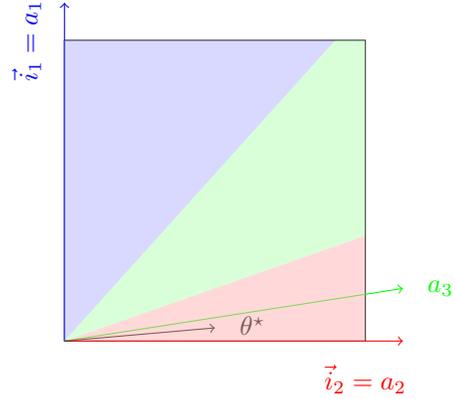
\begin{figure}[ht]
\centering
\begin{tikzpicture}
  \draw (0,0) -- (4,0) -- (4,4) -- (0,4) -- cycle;
  \draw[->,blue] (0,0) -- (0,4.5);
  \draw[->,red] (0,0) -- (4.5,0);
  \draw[->,green] (0,0) -- (4.5,0.7);
  \draw[->] (0,0) -- (2,0.175);
  \node[] at (2.5,0.2) {$\theta^\star$};
  \fill[blue!30,opacity=0.5] (0,0) -- (3.6,4) -- (0,4) -- cycle;
    \fill[green!30,opacity=0.5] (0,0) -- (3.6,4) -- (4,4) -- (4,1.4) -- cycle;
    \fill[red!30,opacity=0.5] (0,0) -- (4,0) -- (4,1.4) -- cycle;`
    \node[rotate=90, text=blue] at (-0.5,4) {$\Vec{i}_1=a_1$};
    \node[text=red] at (4,-0.5) {$\Vec{i}_2=a_2$};
    \node[text=green] at (5,0.7) {$a_3$};
\end{tikzpicture}
\caption{Note that on this figure, $a_2$ and $a_3$ are optimal, alghough playing both of them for the duel will provide feedback of very low value since they lie in the same region of $\R^d$. This is why a duel between $a_1$ and another arm is of greater interest in our exploration setup: sampling \textit{good arms} is not the optimal strategy, hence the link with pure exploration.}
\end{figure}

\begin{restatable}{lemma}{lemmaconcentrationmleestimator}\label{lemma_concentration_mle_estimator} Under \Cref{assumption:boundedness}, for any $\delta \in (0,1)$, with probability at least $1-\delta$, we have that 
\begin{align*}
    & \|\pmle- \theta^\star\|_{V_{t}} \leq \\
    & 20 \left[\sqrt{2 \log (1/\delta) + d\, \log\left(\lambda^{1-1/d}+4t/d\lambda^{1/d} \right)} + \sqrt{\lambda} \right] \eqsp.
\end{align*}
\end{restatable}

\Cref{lemma_concentration_mle_estimator} allows to control the gap between the true reward parameter $\theta^\star$ and the estimation $\hat{\theta}_{T+1}$. We postpone the proof to \Cref{appendix:proofs}.

\textit{Decomposition of the Regret.} Since $\regret$ can be written $\regret(T, (\cA_n)_{n \in [N]}, \theta^\star) = \langle \theta^\star, a^\star_{n^\star} - \hat{a}_{T+1}(\cA_{n^\star})\rangle$ with $a^\star_n = \argmax_{a \in \cA_n}\langle \theta^\star, a\rangle$ and $n^\star = \argmax_{n \in [N]} \langle \theta^\star, a^\star_n - \hat{a}_{T+1}(\cA_{n^\star})\rangle$, we have that
\begin{align}
    \regret(T, (\cA_n)_{n \in [N]}, \theta^\star) & \label{equation:decomposition_regret} \leq \langle \theta^\star, a^\star_{n^\star} -\hat{a}_{T}(\cA_{n^\star})\rangle \\
    & \nonumber \quad - \langle \hat{\theta}_{T+1}, a^\star_{n^\star} - \hat{a}_{T}(\cA_{n^\star})\rangle \\
    & \nonumber = \langle \theta^\star-\hat{\theta}_{T+1}, a^\star_{n^\star} -\hat{a}_{T}(\cA_{n^\star}) \rangle \\
    & \nonumber \leq \|\theta^\star - \hat{\theta}_{T+1}\|_{V_{T+1}} \\
    & \nonumber \quad \times \|a^\star_{n^\star} - \hat{a}_T(\cA_{n^\star})\|_{V_{T+1}^{-1}} \\
    & \nonumber \leq \|\theta^\star - \hat{\theta}_{T+1}\|_{V_{T+1}} \max_{b \in \cB} \|b\|_{V_{T+1}^{-1}}
    \eqsp,
\end{align}
where the third line holds thanks to Hölder inequality. Let $\pi \colon \cB \to [0,1]$ be a distribution over the set of actions. We define the application $g \colon \Delta(\cB)\to \R$ and the design matrix $\matopt(\pi)$ for the distribution $\pi$ as
\begin{equation}\label{equation:definition_V_and_g}
    \matopt(\pi) = \sum_{b \in \cB} \pi(b) \, b \, b^T \; \text{and} \; g(\pi) = \max_{b \in \cB} \norm{b}^2_{\matopt(\pi)^{-1}} .
\end{equation}

A design $\pi$ for our problem is a probability distribution over the set of actions $\cB$. An optimal design $\pi^\star$ is a solution in the \textit{Kiefer-Wolfowitz} theorem while the distribution $\hat{\pi}$ over $\cB$ is a $(1+\epsilon)$-approximation of $\pi^\star$ if $g(\hat{\pi}) \leq (1+\epsilon)g(\pi^\star)$.
\newline

\begin{restatable}{theorem}{theoremsimpleregretupperbound}\label{theorem:simple_regret_upper_bound}
    Let $\epsilon >0$ and suppose that we collect at least $T\geq d^2$ samples according to an $(1+\epsilon)$ approximation $\hat{\pi}$ of the optimal design policy $\pi^\star$ for the problem. Then, for any $\cB$ and $\theta^\star \in \oball(0,1)$, with probability at least $1-\delta$, $\delta \in (0,1)$, we have that
\begin{align}\label{equation:regret_bound_theorem}
    &\regret(T, (\cA_n)_{n \in [N]}, \theta^\star) \leq 20 (1+\epsilon)\, \sqrt{d/T} \ \times \\
    & \nonumber \left[\sqrt{2 \log (1/\delta) + d\, \log\left(\lambda^{1-1/d}+4T/d\lambda^{1/d} \right)} + \sqrt{\lambda} \right] \eqsp.
    \end{align}
\end{restatable}

\begin{restatable}{corollary}{simpleregretupperbound}\label{corollary:simple_regret_upper_bound}
    Suppose that we have selected the samples $\dselect$ to label under $\hat{\pi}$, a $3/2$-approximation of the optimal design policy $\pi^\star$. Choosing $\lambda = 1/d$ for the regularization, for any $\delta \in (0,1)$, under the conditions of \Cref{theorem:simple_regret_upper_bound}, with probability at least $1-\delta$, we have that
    \begin{align*}
        & \regret(T, (\cA_n)_{n \in [N]}, \theta^\star) \leq 30\sqrt{d/T} \times \\
        & \quad \left[\sqrt{2 \log(1/\delta) + d\log((1+4T)/d^{1-1/d})}+1/\sqrt{d} \right] \eqsp,
    \end{align*}
    and as a consequence, choosing $\delta = d^{1-1/d}/(4T+1)$, we can bound the expectation of the regret as
    \begin{align*}
        \E[\regret(T, (\cA_n)_{n \in [N]}, \theta^\star)] &\leq 30 \frac{d+2}{\sqrt{T}}\sqrt{\log\left(\frac{4T+1}{d^{1-1/d}}\right)} \\
        & \quad + 31/\sqrt{T}  \eqsp.
    \end{align*}
\end{restatable}

The lower bound in \Cref{section:lower_bound} matches the upper bound in \Cref{corollary:simple_regret_upper_bound} up to constant or logarithmic factors. Our procedure achieves optimal selection without online feedback, simplifying dataset design and reducing computational costs, as key components like the design matrix or likelihood are computed only once.

\subsection{Changing Action Set}

\textit{Setup.} The term \textit{online} as opposed to \textit{offline} is not always clear in the RLHF literature. We can consider a version of our process with \textit{changing action sets}: at each step $n \in [N]$, where $N \in \N$, a prompt $x_n$ is drawn from the set of prompts $\cX$, and the initial language model $\phi_0$ generates $K$ completions $y_{n}^1, \ldots, y_{n}^K$ associated with $x_n$. At each iteration, from this set of $K$ completions, two samples, $Y_n^{(1)}$ and $Y_n^{(2)}$, must be selected for evaluation by a human labeler. Given that rating samples is costly, the goal is to design an algorithm that optimally selects $T$ pairs, $(x_t, Y_t^{(1)})$ and $(x_t, Y_t^{(2)})$ at each step, to form the most effective dataset before training the reward model while keeping $T$ relatively small. Here, note that $T = N$ based on the previous notation from the offline setting.

\textit{Objectives and metrics.} As before, we are interested in the precision of the reward estimation after the $T$ steps, which guides us towards selecting an optimal dataset $\dselect$. This is why we keep the same objective as before and thus consider the \textit{simple regret} of the procedure, defined here as
\begin{equation}\label{equation:definition_online_regret}
    \regret(T, (\cA_n)_{n \in [N]}, \theta) = \max_{n \in [N]} \max_{a \in \cA_n} \langle \theta, a - \hat{a}(\cA_n)\rangle \eqsp,
\end{equation}
where $\hat{a}(\cA_n) \in \cA_n$ is the best-arm prediction of the procedure among the set $\cA_n$. Considering the same kind of objective also allows us to compare both kinds of procedures.

Formally, at each step $t \in [T]$, an action set $\cA_t, t \in [T]$ is provided and the algorithm must select a pair of samples $(\al_t, \am_t) \in \cA_t^2$. We can define the learner's history as $\cH_t = \sigma(\{X_s, \al_s, \am_s\}_{s=1, \ldots, t})$, $\cH_0 = \varnothing$ and the learner uses an algorithm $\alg$ to choose the action pair $(\al_t, \am_t)$ based on $(\cH_{t-1}, \cA_t, U_t)$, $(U_t)_{t \in \N^\star}$ being a family of independent uniform random variables on $[0,1]$ allowing randomization in $\alg$. A commonly proposed strategy to explore is to pick the pair $(\al_t, \am_t) \in \cA_t^2$ at each iteration following
\begin{equation*}
\al_t, \am_t \in \argmax_{a,a' \in \cA_t} \|a-a'\|_{\desmat_t^{-1}} \eqsp.
\end{equation*}

However, such a strategy, as well as any procedure $\alg$ for a setup with changing arms cannot converge, since we do not make i.i.d. assumptions throughout this work and allow adversarial action sets. The choice of specific action sets that prevent convergence for any algorithm is explained in the proof of \Cref{lemma:lower_bound_online}.

\begin{restatable}{theorem}{lowerboundonline}\label{lemma:lower_bound_online}
Consider the $2$-dimensional euclidian space Span$(e_1, e_2)$ as the whole action space. In that case, there exists a set of actions $(\cA_t)_{t \in [T]}$ and some $\theta^\star \in \oball(0,1)$ such the regret defined in \eqref{equation:definition_online_regret} for any algorithm $\alg$ satisfies
\begin{equation*}
    \regret_\alg(T, (\cA_t)_{t \in [T]}, \theta^\star) \geq \rme^{-c}/2 \eqsp,
\end{equation*}
for some $c >0$ independent of $T$ and $d$.
\end{restatable}

\subsection{Extensions of the Offline Scenario}

An online version of our setup could involve gathering the entire action set $\dini$ before the learner iteratively selects samples from $\dini$, using feedback at each step to inform the next selection.

However, it may be unrealistic to assume online feedback at each step. A more feasible approach would involve sampling a \textit{batch of pairs} offline, sending them to labelers for evaluation, and then selecting the next batch based on the feedback from the entire batch of preferences. This setup is more practical, as it allows a dynamic process where labelers provide preferences for a batch of pairs before the next batch is sampled, although they do not have to provide feedback at each step (too complex in practice). The efficiency of the method depends on the batch size - a batch of size $1$ mirrors the \online setting and a batch of size $T$, which recovers our offline setup.

A lot of approaches in practice use pairs sampled offline, followed by training \texttt{DPO} \citep[][]{rafailov2024direct} on the resulting dataset. Our offline strategy, allowing the selection of pairs that are statistically the most informative, could significantly enhance \texttt{DPO}'s performance, without requiring too many additional computations. It would be interesting to investigate whether the optimal dataset selection depends on the reward modeling or the choice of the $\psi$-function within the particular setup of \citet{azar2024general}.

Another potential sampling strategy of the dataset could be to select points based on a binary search procedure in $\R^d$ \citep{lobel2018multidimensional}, combined with bandit feedback. The reward parameter can be efficiently identified through cuts of the unit ball in $\R^d$, which is why this approach is of some interest. Finally, we believe that our work has strong connections with best-arm identification in linear bandits \citep{soare2014best, degenne2020gamification}, particularly where ideas from \textit{Optimal Design} are applied. The challenge arises from the absence of online feedback in our setup, and one should look for relaxations in order to circumvent this issue.
\section{Lower Bound}\label{section:lower_bound}

As before, after $N \in \N$ sets of prompt-completions have been sampled, we consider again $\cB = \{a_n^i - a_n^j\}_{n \in [N], i, j \in [K]} = \cup_{n \in [N]} \{\cA_n - \cA_n\}$, the set of all possible action pairs. We consider an algorithm $\alg$ that samples $T$ pairs from this set and receives feedback $Y_t$ from a labeler. For any $t \in [T]$, we write $(\al_t, \am_t)$ for the pair of actions sampled by $\alg$ at iteration $t$, even though there is no "time ordering" in our procedure. This choice corresponds to taking an action
\begin{equation*}
B_t \in \cB \colon B_t = \al_t - \am_t \eqsp.
\end{equation*}
Formally, $\alg$ selects an action $B_t$ at step $t$ and observes a logistic feedback $Y_t \sim \ber(\sigma(\theta^T B_t))$ for reward parameter $\theta$. Based on $(Y_t)_{t \in [T]}$, $\alg$ estimates $\theta$ and then for any input $\cA_n, n \in [N]$, plays some action $\hat{a}_T(\cA_n) \in \cA_n$. The performance is still evaluated with the simple regret
\begin{equation*}
    \regret(T, (\cA_n)_{n \in [N]}, \theta) = \max_{n \in \cA_n} \langle \theta, a^\star_n - \hat{a}(\cA_n)\rangle \eqsp.
\end{equation*}
We write $\P_{\theta}$ for the probability distribution of our linear contextual bandit instance with reward parameter $\theta$ over the whole possible action space $\cA$ and $(P_{b_1}^{\theta}, \ldots, P_{b_L}^{\theta})$ the probability distribution associated with the different pairs of arms from $\cB$ with the parameter $\theta$. $\P_{\theta'}$ as well as $(P_{b_1}^{\theta'}, \ldots, P_{b_L}^{\theta'})$ stand for the same objects with parameter $\theta'$. Within this setup, we can use the divergence bound for general spaces from \citet[][15.8]{lattimore2020bandit}, and write
\begin{align}\label{equation_div_expression}
    & \kldiv(\P_{\theta}, \P_{\theta'}) = \sum_{t=1}^T \E_{\theta}[\kldiv(P_{\al_t -\am_t}^{\theta}, P_{\al_t - \am_t}^{\theta'})] \\
    & \quad \nonumber = \sum_{t=1}^T \E_{\theta}[\kldiv(\ber(\sigma(\theta^T B_t)), \ber(\sigma(\theta^{' T}B_t))] \eqsp.
\end{align}

\begin{restatable}{lemma}{inequalitydivergences}\label{lemma:inequality_divergences}
    Assume that $\P$ and $\Q$ are probability measures on a measurable space $\cX, \cA$ such that $\P$ is absolutely continuous with respect to $\Q$. Then
    \begin{equation*}
        \kldiv(\P, \Q) \leq \log(1 + \chidiv(\P, \Q)) \leq \chidiv(\P, \Q) \eqsp.
    \end{equation*}
    If $\P \ll \Q$ does not hold, then the result is trivial.
\end{restatable}
The result of this lemma is of great help to upper bound the divergence of our Bernoulli random variables since the $\chi^2$ divergence is easier to use with Bernoulli distributions. The proof is postponed to appendix A. 

We now present our main theorem from this section, which gives a lower bound for our setup whiwh matches our upper bound from \Cref{corollary:simple_regret_upper_bound} up to constant or logarithmic factors.

\begin{restatable}{theorem}{lowerbound}\label{theorem:lower_bound}
    Suppose that $d \geq 16$ and that $T \geq d^2$. For any algorithm $\alg$ which samples $T$ pairs from $\cB$ and receives a preference feedback before outputing an action $\hat{a}(\cA_n) \in \cA_n$ for an input $\cA_n$, there exists $(\cA_n)_{n \in [N]} \subseteq \oball(0,1)$ as well as $\theta^\star \in \oball(0,1)$ such that
    \begin{equation*}
        \regret(T, (\cA_n)_{n \in [N]}, \theta^\star) \geq d \, \rme^{-5}/4\sqrt{T} \eqsp.
    \end{equation*}
\end{restatable}
The self-concordance property of the sigmoid function \citep{bach2010self} is of some importance in this bound since its properties play a role inequalities we work this. This lower bound allows us to claim that our proposed method is close from optimality. To our knowledge, our work is the first contribution with such an evidence of optimality for preferences dataset selection in RLHF.

\section{Related Work}

The extraordinary capabilities of RLHF to fine-tune large language models \citep{brown2020language, bubeck2023sparks, casper2023open} are a reason for the growing attention in the field. The modeling of RLHF as Markov Decision Processes \citep{wang2023rlhf} or bandits is a common assumption and has been proposed for various goals \citep{zhu2024iterative, mehta2023sample}. Our theoretical work involves comparing the feedback given as one action preferred to another, which is directly related to the problem of dueling bandits \citep{yue2012k, gabillon2012best, sui2018advancements}. Some foundations have been laid out to study preference-based and dueling bandits or RL \citep{pacchiano2021dueling, novoseller2020dueling} while other works consider offline RL where the learning does not result from interactions with the environment \citep{zhan2023provable}. As mentioned in the latter, an issue in offline RL is the insufficient coverage of the space with the collected data. Interestingly, this is an issue that we address here through optimal design.

A lot of works compare online to offline RLHF \citep{hu2023aligning, tang2024understanding, cen2024value} but here, we make the \textit{online-offline} distinction for the dataset generation before the reward modeling and policy optimization start; something that always needs to be done offline in practice.

The human feedback in RLHF is usually given as a preference between a pair of generations associated with a same context, hence the link with with contextual dueling bandits \citep{dudik2015contextual}. Our setting considers a binary feedback, which relates it to generalized linear bandits \citep{filippi2010parametric} and more precisely to logistic bandits \citep{lee2024nearly, lee2024improved, abeille2021instance, faury2020improved}. \citet{bengs2022stochastic} provide an interesting setup for contextual dueling bandits but do not provide all the proofs and rely on the work by \citet{vaswani2019old} to bound the distance between their estimate and the true reward vector, although the bounds of \citet{vaswani2019old} hold for an ordinary least square estimator and not a maximum likelihood estimator - a harder task because of the lack of explicit form of the estimator. \citet{saha2021optimal} propose a very interesting approach to transform a contextual dueling bandit setting into a linear contextual linear bandit setting but leverage the iterative structure of the problem to do so, which is not possible in our case. Finally, \citet{gabillon2012best} give important ideas around pure exploration since it is one of the first and most important works in best-arm identification for multi-armed bandits.

A lot of empirical works have been done around the problem of active learning and optimal choice of samples for diverse and informative collection of data \citep{metz2023rlhf}. In an online setup, \citet{chen2024optune} propose an interesting approach to improve the alignment with a reweighing of the generations to improve the collected information while some theoretical foundations have been have already been laid out by \citet{lindner2023algorithmic, wang2023rlhf}. Our work has hope to stand at the crossroad of algorithmic foundations and practical considerations. There are already seminal works in new directions which involve working with off-policy evaluation for preference learning \citep{bhargava2024off} or active learning for choosing teaching examples \citep{wang2021teaching}. The few theoretical attempts in our direction \citep{das2024provably, ji2024reinforcement} consider an active learning setting where the selection of the sample pairs is done concomitantly with the received feedback (\textit{online setup}), which is an unrealistic assumption due to the practical operation of the work with human labelers. Also, they propose interesting methods but without lower bounds nor a deep theoretical analysis.

More generally, recent works consider learning from human preferences, such as \citet{mukherjee2024optimal} where the preference ordering over a list is learnt or \citet{munos2023nash}, where a Nash equilibrium is learnt. Instead of learning the preferences based on a score, they can be learnt with the data being some preference pairs, hence the link with the dueling bandit framework as some theoretical model \citep{yan2022human}.

Finally, we propose a lower bound for our setup. Such bounds already exist for dueling bandits but rely on the sequential structure of the problem and a global regret objective \citep{saha2021optimal, yue2012k, komiyama2015regret}, something that we cannot do with the simple regret that we are looking for. This is why we see our problem as a logistic bandit and go back to the traditional Bretagnolle-Huber inequality \citep{bretagnolle1979estimation} to control the \textit{bad} events and obtain the lower bound for our .
\section{Conclusion}

This paper addresses the problem of selecting pairs of language model generations to present to labelers in order to maximize the information gathered from their feedback. The goal is to develop an efficient strategy for selecting which generations - or \textit{arms} - should be rated to retrieve the most valuable information before fine-tuning the model. To tackle this, we build on the framework of pure exploration in linear contextual dueling bandits, a well-suited approach for the specific task that we are looking for. We operate under several key assumptions: a linear reward; the Bradley-Terry model that governs the preferences between pairs and the boundedness of the action set as well as the reward parameter.

The core of our approach lies in leveraging optimal design techniques, which allow us to strategically choose the \textit{arms}. By doing so, we maximize the information gained from each comparison, making the rating process highly efficient. Furthermore, by applying information-theoretic tools, we derive a lower bound for the performance of our method. Remarkably, this lower bound matches our upper bound up to constant and logarithmic factors, thereby demonstrating the optimality of our approach.

Finally, we highlight that the techniques developed in this work are not only theoretical but also closely related to practical methods used for selecting the pairs and applying RLHF. The results suggest that our procedure can be both practical and highly effective, offering a significant advancement in how LM generations are selected before receiving human feedback preferences.
\onecolumn
\twocolumn
\section*{Acknowledgements}
Funded by the European Union (ERC, Ocean, 101071601). Views and opinions expressed are however those of the author(s) only and do not necessarily reflect those of the European Union or the European Research Council Executive Agency. Neither the European Union nor the granting authority can be held responsible for them.

\bibliographystyle{apalike}
\bibliography{sample}


\newpage
\onecolumn
\appendix

\section{Proofs}
\label{appendix:proofs}

\begin{restatable}{lemma}{boundingsigmoid}\label{lemma:bounding_sigmoid}
    For any $x \in \R, \sigma'(x)>0$ and for any interval of the form $[-\alpha, \beta]$ for $\alpha, \beta >0$, we have that $\sigma'$ is increasing on $[-\alpha, 0]$ and decreasing on $[0, \beta]$.
\end{restatable}
\begin{proof}[Proof of \Cref{lemma:bounding_sigmoid}]
    Note that for any $x \in \R$ we have
\begin{align*}
    \sigma'(x) = e^{-x}/(1+e^{-x})^2 \;\;\; \text{ and }\;\;\;\sigma''(x) = e^{-x} (e^{-x}-1)/(1+e^{-x})^3 \eqsp,
\end{align*}
and we observe that $\sigma''$ cancels out in $0$, is positive on $\R_-^\star$ and negative on $\R_+^\star$, hence the result.
\end{proof}

\differentiationequationlikelihood*
\begin{proof}[Proof of \Cref{lemma:differentiation_equation_likelihood}]
    As mentioned in the main text, a direct computation gives that
    \begin{align*}
    \cL_t(\{\al_s, \am_s, Y_s\}_{s \in [t-1]}, \theta) & = \sum_{s=1}^{t-1} \left\{Y_s \log(\sigma(\langle \theta, \al_s-\am_s\rangle)) + (1- Y_s)\log(\sigma(-\langle \theta, \al_s-\am_s\rangle))\right\} - \lambda \|\theta\|_2^2 /2 \\
    & = \sum_{s=1}^{t-1}\left\{ Y_s \log\left(\frac{1}{1+\rme^{-\langle \theta, \al_s-\am_s\rangle}}\right) + (1- Y_s)\log\left(\frac{1}{1+\rme^{\langle \theta, \al_s-\am_s\rangle}} \right)\right\} - \lambda \|\theta\|_2^2 /2 \eqsp,
\end{align*}
and \Cref{lemma:bounding_sigmoid} offers an expression for the differential of the sigmoid, which gives that
\begin{align*}
    \nabla_{\theta} \cL_t(\{\al_s, \am_s, Y_s\}_{s \in [t-1]}, \theta) &= \sum_{s=1}^{t-1} \left\{ Y_s \frac{(\al_s-\am_s)\rme^{-\langle \theta, \al_s-\am_s\rangle}(1+\rme^{-\langle \theta, \al_s-\am_s\rangle})^{-2}}{(1+\rme^{-\langle \theta, \al_s-\am_s\rangle})^{-1}}\right\} \\
    & \quad + \sum_{s=1}^{t-1}\left\{-(1-Y_s) \frac{(\al_s-\am_s)\rme^{\langle \theta, \al_s-\am_s\rangle}(1+\rme^{\langle \theta, \al_s-\am_s\rangle})^{-2}}{(1+\rme^{\langle \theta, \al_s-\am_s\rangle})^{-1}} \right\} - \lambda \theta \\
    & =  \sum_{s=1}^{t-1}\left\{Y_s(\al_s-\am_s)\frac{\rme^{-\langle \theta, \al_s-\am_s\rangle}}{1+ \rme^{-\langle \theta, \al_s-\am_s\rangle}} +(Y_s-1)(\al_s-\am_s)\frac{\rme^{\langle \theta, \al_s-\am_s\rangle}}{1+ \rme^{\langle \theta, \al_s-\am_s\rangle}} \right\} - \lambda \theta \\
    & = \sum_{s=1}^{t-1}(\al_s-\am_s)\left\{Y_s\frac{1}{1+ \rme^{\langle \theta, \al_s-\am_s\rangle}} +(Y_s-1)\frac{1}{1+ \rme^{-\langle \theta, \al_s-\am_s\rangle}}  \right\} -\lambda \theta \\
    & = \sum_{s=1}^{t-1}(\al_s-\am_s)\left\{Y_s(1-\P_{\theta}(Y_s=1)) + (Y_s-1)\P_{\theta}(Y_s=1) \right\} - \lambda \theta \\
    & =  \sum_{s=1}^{t-1}(\al_s-\am_s)(Y_s -\P_{\theta}(Y_s=1))- \lambda \theta \eqsp,
\end{align*}
hence the result. By definition of the maximum likelihood estimator, $\hat{\theta}_t$ satisfies $\nabla_{\theta} \cL_t(\{\al_s, \am_s, Y_s\}_{s \in [t-1]}, \hat{\theta_t})=0$, and therefore we obtain \eqref{equation:solution_hat_theta}.
\end{proof}

\lemmaconcentrationmleestimator*
\begin{proof}[Proof of \Cref{lemma_concentration_mle_estimator}]
    Our proof is inspired by results from \citet{di2023variance, ji2024reinforcement}. For any $t \in [T+1]$, recall the definition of $H_t$; we now define
    \begin{align*}
        & X_t = Y_t - \P(Y_t=1) = \1(\al_t \succ \am_t) - \sigma(\langle \theta^\star, \al_t-\am_t\rangle) \;\; \text{  and  } \;\; Z_t = \sum_{s=1}^{t-1} X_s (\al_s - \am_s) \eqsp. 
    \end{align*}
As we saw in \eqref{equation:H_for_hattheta}, by definition of the maximum likelihood estimator, and $\theta^\star$ as the true reward parameter, we have that
\begin{align*}
    & H_t(\thmle) = \sum_{s=1}^{t-1} Y_s(\al_s-\am_s) \; \; \text{  and  } \; \;  H_t(\theta^\star) = \lambda \theta^\star + \sum_{s=1}^{t-1}\P(Y_s=1)(\al_s-\am_s) \eqsp,
\end{align*}
where $\P$ stands for the true reward distribution, according to the parameter $\theta^\star$. Therefore
\begin{equation}\label{cqslkjgsqlrk}
    H_t(\thmle) - H_t(\theta^\star) = \sum_{s=1}^{t-1}[Y_s-\P(Y_s=1)](\al_s-\am_s) - \lambda \theta^\star = Z_t - \lambda \theta^\star \eqsp.
\end{equation}
We now consider the difference $H_t(\theta_1) - H_t(\theta_2)$ for arbitrary $\theta_1, \theta_2$ in $\R^d$, and apply a first order Taylor expansion with integral remainder to the function $\theta \mapsto \sigma(\langle \theta, \al_s-\am_s\rangle)$ on the space $\R^d$ in each term of the sum, which leads to
\begin{align*}
    H_t(\theta_1) - H_t(\theta_2)
    & = \lambda \theta_1 - \lambda \theta_2 + \sum_{s=1}^{t-1}(\sigma(\theta_1^T(\al_s - \am_s))-\sigma(\theta_2^T(\al_s - \am_s)))(\al_s-\am_s) \\
    & = \lambda (\theta_1 - \theta_2) + \sum_{s=1}^{t-1} (\al_s-\am_s) \int_{u=0}^1 (\al_s-\am_s)^T\sigma'(\langle \theta_1 + u(\theta_2-\theta_1), \al_s-\am_s\rangle)(\theta_1-\theta_2) du \\
    & = \left[\lambda I + \sum_{s=1}^{t-1} (\al_s-\am_s) (\al_s-\am_s)^T \int_{u=0}^1 \sigma'(\langle \theta_1 + u(\theta_2-\theta_1), \al_s-\am_s\rangle) du \right](\theta_1-\theta_2) \eqsp.
\end{align*}
We define for any $t \in [T], \theta_1, \theta_2 \in \R^d$
\begin{equation*}
    G_t(\theta_1, \theta_2) = \lambda I + \sum_{s=1}^{t-1} (\al_s-\am_s) (\al_s-\am_s)^T \int_{u=0}^1 \underbrace{\sigma'(\langle \theta_2 + u(\theta_1-\theta_2), \al_s-\am_s\rangle)}_{\geq 0} du \succ 0 \eqsp,
\end{equation*}
since $\lambda$ is chosen such that $\lambda >0$. Note that for any $t, \theta_1, \theta_2$: $G_t(\theta_1, \theta_2)$ is symmetric. By definition, we have that for any $\theta_1, \theta_2 \in \R^d$
\begin{align*}
    H_t(\theta_1) - H_t(\theta_2) = G_t(\theta_1, \theta_2) (\theta_1 - \theta_2) \eqsp,
\end{align*}
which gives that
\begin{equation}\label{sfzgf}
    \|\theta_1 - \theta_2\|_{G_t(\theta_1, \theta_2)} = \sqrt{ (\theta_1, \theta_2)\, G_t \, G_t^{-1} \, G_t \,(\theta_1, \theta_2)} = \|H_t(\theta_1) - H_t(\theta_2)\|_{G_t^{-1}(\theta_1, \theta_2)} \eqsp.
\end{equation}

Using \Cref{lemma:bounding_sigmoid}, since $\sigma'(-2) \geq 0.1$ and $\sigma'(2)\geq 0.1$, we have that for any $x \in [-2, 2], \sigma'(x) \geq 0.1$. Note that for any $s \in [T], \theta \in \oball(0,1), \langle \theta, \al_s-\am_s\rangle \in [-2, 2]$. Therefore, under \Cref{assumption:boundedness}, a convexity argument gives that for any $\theta \in \oball(0,1), u\in [0,1], \ps{\theta + u(\theta^\star-\theta)}{\al_s-\am_s} \in [-2,2]$, and we obtain
\begin{align*}
    \int_{u=0}^1 \sigma'(\langle \pmle + u(\theta^\star-\pmle), \al_s-\am_s\rangle) du \geq 0.1 \eqsp,
\end{align*}
since $\pmle \in \oball(0,1)$. Note that we use here the fact that $\pmle$ lies in the unit ball to control the boundedness of the sigmoid function. This leads to
\begin{equation}\label{qsdjgs}
    V_t \prec 10 \, G_t(\theta^\star, \pmle) \;\; \text{  and  } \;\; G_t(\theta^\star, \pmle)^{-1} \preceq 10 \left(\lambda I + \sum_{s=1}^{t-1}(\al_s - \am_s) (\al_s-\am_s)^T \right)^{-1} = 10 \, V_{t}^{-1} \eqsp,
\end{equation}
and we obtain
\begin{align*}
    \|\theta^\star -\pmle\|_{V_{t}} & \leq \sqrt{10}\|\theta^\star - \pmle\|_{G_t(\theta^\star, \pmle)} \\
    & = \sqrt{10} \, \|H_t(\theta^\star) - H_t(\pmle) \|_{G_t^{-1}(\theta^\star, \pmle)} \\
    & \leq 10 \|H_t(\pmle)-H_t(\theta^\star)\|_{V_t^{-1}} \\
    & \leq 10(\|H_t(\hat{\theta}_t)-H_t(\theta^\star)\|_{V_t^{-1}} + \|H_t(\pmle)-H_t(\hat{\theta}_t)\|_{V_t^{-1}}) \\
    & \leq 20 \|H_t(\hat{\theta}_t) -H_t(\theta^\star)\|_{V_t^{-1}} \\
    & = 20 \norm{Z_t - \lambda \theta^\star}_{V_t^{-1}} \\ & \leq 20 (\|Z_t\|_{V_t^{-1}} + \sqrt{\lambda}) \eqsp,
\end{align*}
where the second line holds by \eqref{sfzgf}, the third line by \Cref{qsdjgs}, the fourth by the triangular inequality, the fifth by definition of $\pmle$ \eqref{equation:definition_projected_mle} and the penultimate by \eqref{cqslkjgsqlrk}. Note that for any $t \in [T]$, $X_t \in [-1,1]$ and is therefore $1$-subgaussian by Hoeffding inequality. Therefore, we apply \citet[][Theorem 1]{abbasi2011improved}, which gives that with probability at least $1-\delta$, we have
\begin{align*}
    \norm{Z_t}_{V_t^{-1}}^2 = \left\|\sum_{s=1}^{t-1} X_s (\al_s-\am_s) \right\|^2_{V_{t}^{-1}} &\leq 2 \log \left( \sqrt{\det(V_{t} )} / (\sqrt{\det(V_0)} \delta)\right) \leq 2 \log(1/\delta) + \log( \det V_{t}/\det V_0) \eqsp.
\end{align*}
By definition, we have that for any $t \geq 1$
\begin{align*}
    V_{t} = V_0 +\sum_{s=1}^{t-1} (\al_s - \am_s)(\al_s - \am_s)^T \eqsp,
\end{align*}
and now using \citet[][Lemma 19.4]{lattimore2020bandit}, we can write
\begin{align*}
    \log( \det V_{t}/\det V_0) \leq d \log\left(\frac{\TR(V_0) + 4 (t-1)}{d \det(V_0)^{1/d}}\right) = d \log\left(\frac{d \lambda +4(t-1)}{d \lambda^{1/d} }\right) \eqsp.
\end{align*}
We finally obtain that with probability at least $1-\delta$
\begin{align*}
    \|\pmle- \theta^\star\|_{V_{t}} \leq 20 \left[\sqrt{2 \log (1/\delta) + d\, \log\left(\lambda^{1-1/d}+4t/d\lambda^{1/d} \right)} + \sqrt{\lambda} \right] \eqsp,
\end{align*}
hence the result.
\end{proof}

\theoremsimpleregretupperbound*
\begin{proof}[Proof of \Cref{theorem:simple_regret_upper_bound}]
The condition $T\geq d(d+1)/2$ ensures that we collect enough points so that the optimal design policy $\hat{\pi}$ satisfies the results from \Cref{theorem:kiefer_wolfowitz}. Using the decomposition of the regret \eqref{equation:decomposition_regret} and \Cref{lemma_concentration_mle_estimator}, we obtain that with a probability at least $1-\delta$
\begin{align}\label{equation:proof_simple_regret_sfkg}
    \regret(T,(\cA_n)_{n \in [N]}, \theta^\star) \leq 20 \left[\sqrt{2 \log (1/\delta) + d\, \log\left(\lambda^{1-1/d}+4 T /d\lambda^{1/d} \right)} + \sqrt{\lambda} \right]
    \max_{(a, a') \in \dini^2}\norm{a-a'}_{V_{T+1}^{-1}} \eqsp.
\end{align}
Let $\hat{\pi}$ be an $1+\epsilon$ approximation of the optimal design policy $\pi^\star$. For any distribution $\pi$, we have $\matopt(\pi) = \sum_{b \in \cB} \pi(b) \, b \, b^T$. The regularized design matrix based on the collected samples from $\hat{\pi}$ is defined as $V_{T+1} = \lambda I + \sum_{t=1}^T B(t)\, B(t)^T = \lambda I + \sum_{b \in \cB} \lceil T \, \hat{\pi}_b\rceil b \, b^T$ with $B(t) = \al_t - \am_t$, since the samples $(\al_t, \am_t)_{t \in [T]}$ are chosen according to $\algodpo$. Therefore, for any $b \in \cB$, we have that
\begin{align*}
    \norm{b}_{V_{T+1}^{-1}}^2 &= b^T \left(\lambda I + \sum_{\Tilde{b} \in \cB} \lceil T \, \hat{\pi}_{\Tilde{b}}\rceil \Tilde{b} \, \Tilde{b}^T\right)^{-1} b \\
    & \leq b \left( \sum_{\Tilde{b} \in \cB} T \, \hat{\pi}_{\Tilde{b}} \Tilde{b} \, \Tilde{b}^T\right)^{-1} b \\
    & = \frac{1}{T} \, b^T \left(\sum_{\Tilde{b} \in \cB} \, \hat{\pi}_{\Tilde{b}} \Tilde{b} \, \Tilde{b}^T\right)^{-1} b \\
    & = \frac{1}{T} \, \norm{b}_{\matopt^{-1}(\hat{\pi})}^2 \\
    & = (1+\epsilon)\, d \, / \, T \eqsp,
\end{align*}
where the last line holds thanks to \Cref{algorithm:frank_wolfe_algorithm} and results on its convergence \citep[see, e.g., ][21.2]{lattimore2020bandit}. It gives that
\begin{align*}
    \max_{(a,a')\in \dini}\norm{a-a'}_{\desmat_{T+1}^{-1}} \leq \sqrt{(1+\epsilon)\, d \, / \, T} \eqsp,
\end{align*}
and plugging this bound in \eqref{equation:proof_simple_regret_sfkg}, we obtain the result.
\end{proof}

\simpleregretupperbound*
\begin{proof}[Proof of \Cref{corollary:simple_regret_upper_bound}]
By \Cref{equation:regret_bound_theorem}, we have that for any $\delta\in (0,1)$
\begin{equation}\label{qzgqsgqerh}
    \regret(T, (\cA_n)_{n \in [N]}, \theta^\star) \leq 20 (1+\epsilon)\, \sqrt{d/T} \left[\sqrt{2 \log (1/\delta) + d\, \log\left(\lambda^{1-1/d}+4T/d\lambda^{1/d} \right)} + \sqrt{\lambda} \right] \eqsp,
\end{equation}
and we now choose $\hat{\pi}$ to be a $3/2$-approximation of the optimal policy $\pi^\star$ as well as $\lambda = 1/d$. Plugging these quantities in \eqref{qzgqsgqerh} gives
\begin{align*}
     \regret(T, (\cA_n)_{n \in [N]}, \theta^\star) & \leq 20 \, \times 3/2 \, \times \sqrt{d/T} \times \left[\sqrt{2 \log (1/\delta) + d\, \log\left(d^{1/d-1}+4 \, T \, d^{1/d}/d \right)} + 1/\sqrt{d} \right] \\
     & = 30\sqrt{d/T} \times \left[\sqrt{2 \log(1/\delta) + d\log((1+4T)/d^{1-1/d})}+1/\sqrt{d} \right] \eqsp,
\end{align*}
hence the first part of the corollary. Observe that since $\regret(T, (\cA_n)_{n \in [N]}, \theta^\star) = \max_{n \in [N]}\max_{a^\star_n \in \cA_n} \ps{\theta^\star}{a_n^\star- \hat{a}_T(\cA_n)}$, we have that $\regret(T, (\cA_n)_{n \in [N]}, \theta^\star)\leq 2$ under \Cref{assumption:boundedness}. Therefore, for any $\mathrm{C}>0$, we can write
    \begin{equation}\label{sfqgqgrg}
        \E\left[\regret(T, (\cA_n)_{n \in [N]}, \theta^\star)\right] \leq \P\left(\regret(T, (\cA_n)_{n \in [N]}, \theta^\star) \leq \mathrm{C} \right) \cdot \mathrm{C} + 2\cdot (1- \P\left(\regret(T, (\cA_n)_{n \in [N]}, \theta^\star) \leq \mathrm{C}\right)) \eqsp,
    \end{equation}
    and we now apply \eqref{sfqgqgrg} with $\mathrm{C} = 30\sqrt{d/T} \times \left[\sqrt{2 \log(1/\delta) + d\log((1+4T)/d^{1-1/d})}+1/\sqrt{d} \right]$ and $\delta = d^{1-1/d}/(4T+1)$. The first part of the corollary that we already proved gives
    \begin{align*}
        \E\left[\regret(T, (\cA_n)_{n \in [N]}, \theta^\star)\right] & \leq (1- d^{1-1/d}/(4T+1)) \times 30 \, \sqrt{d/T} \times \left[\sqrt{2 \log((4T+1)/d^{1-1/d}) + d\log((1+4T)/d^{1-1/d})}+1/\sqrt{d} \right] \\
        & \quad + 2\times d^{1-1/d}/(4T+1) \\
        & \leq 30 \sqrt{d/T}\left[\sqrt{(d+2)\log\left(\frac{4T+1}{d^{1-1/d}}\right)}+1/\sqrt{d}\right] + d/2T \\
        & \leq 30 \frac{d+2}{\sqrt{T}}\sqrt{\log\left(\frac{4T+1}{d^{1-1/d}}\right)}+30/\sqrt{T} + \frac{d}{2T} \\
        & \leq 30 \frac{d+2}{\sqrt{T}}\sqrt{\log\left(\frac{4T+1}{d^{1-1/d}}\right)} + 31/\sqrt{T} \eqsp,
    \end{align*}
    where the last line holds since $T\geq d(d+1)/2$. Hence the second part of the result.
\end{proof}

\lowerboundonline*
\begin{proof}[Proof of \Cref{lemma:lower_bound_online}]
    Consider the space $\R^2$ with an orthonormal basis $(e_1, e_2)$. Suppose that $\theta^\star \in \Theta = \{\pm e_2\}$ and that the action sets are $\cA_1 = \ldots = \cA_{T-1} = \{\pm e_1\}$ and $\cA_T = \{ \pm e_2\}$. After sampling data and preferences, $\alg$ outputs an action $\hat{a}$ and is then evaluated with the simple regret
    \begin{equation*}
        \regret(T, (\cA_n)_{n \in [N]}, \theta) = \max_{t \in [T]} \max_{a \in \cA_t} \langle \theta, a - \hat{a}(\cA_t)\rangle \eqsp,
    \end{equation*}
    and we consider that $\alg$ sampled arms $\{(\al_t, \am_t)\}_{t \in [T]}$. We write $\al_t-\am_t = B_t$. For any $t \in [T-1], B_t \in \text{Span}(e_1)$. Now consider the events $\{\hat{a} = e_2\}$ and $\{\hat{a} = -e_2\}$. We write $\theta = e_2$ and $\theta' = -e_2$. We have $\max_{a \in \oball_2(0,1)} \langle \theta, a \rangle = \max_{a \in \oball_2(0,1)} \langle \theta', a \rangle = 1$. We now use the Bretagnolle-Huber inequality and write
    \begin{align}\label{qzsgfsrg}
         \P_{\theta}(\{\hat{a} &= -e_2\})+\P_{\theta'}(\{\hat{a} = e_2\}) = \P_{\theta}(\{\hat{a} = -e_2\}) + \P_{\theta'}(\{\hat{a} = -e_2\}^\comp) \geq \exp(-\kldiv(\P_{\theta}, \P_{\theta'}))/2 \eqsp,
    \end{align}
    as well as \citet[][15.8]{lattimore2020bandit} to obtain
    \begin{align*}
        \kldiv(\P_{\theta}, \P_{\theta'}) & = \E\left[\sum_{t=1}^T \kldiv(P^{\theta}_{B_t}, P^{\theta'}_{B_t})\right] \\
        & = \E_{\theta}\left[\sum_{t=1}^{T-1} \kldiv(\ber(\sigma(\theta^T B_t)), \ber(\sigma(\theta'^T B_t))) \right] + \E_{\theta}[\kldiv(\ber(\sigma(\theta^T B_T)), \ber(\sigma(\theta'^T B_T)))] \\
        & = \E_{\theta}[\kldiv(\ber(\sigma(\theta^T B_T)), \ber(\sigma(\theta'^T B_T)))] \eqsp,
    \end{align*}
    since $\theta, \theta' \in$ Span$(e_2), B_t \in$ Span$(e_1)$ for any $t \in [T-1]$, which gives $\theta'^T B_t =  \theta^T B_t =0$. Therefore $\kldiv(\ber(\sigma(\theta^T B_t)), \ber(\sigma(\theta'^T B_t))) = \kldiv(\ber(1/2), \ber(1/2))=0$ for any $t \in [T-1]$. There exists a constant $c>0$ independent of $T$ and the dimension such that $\kldiv(\ber(\sigma(\theta^T B_T)), \ber(\sigma(\theta'^T B_T))) \leq c$. Thus, at least one of the terms in the left-hand side of \eqref{qzsgfsrg} is bigger than $\exp(-c)/4$ - say $\P_{\theta'}(\{\hat{a} = e_2\})\geq \exp(-c)/4$. Which one being bigger than $\exp(-c)/4$ does not matter by symmetry. The regret incurred under $\theta'$ for $\{\hat{a} = e_2\}$ holding is $2$ and we finally obtain that
    \begin{equation*}
        \regret_{\alg}(T,(\cA_t)_{t \in [T]}, \theta') \geq \rme^{-c}/2 \eqsp.
    \end{equation*}
\end{proof}

\inequalitydivergences*
\begin{proof}[Proof of \Cref{lemma:inequality_divergences}]
    By definition of the KL-divergence, we can write
    \begin{equation*}
        \kldiv(\P, \Q) = \int_{\cX} \log\left(\frac{d\P}{d\Q}\right) d\P \eqsp,
    \end{equation*}
    and applying Jensen's inequality with the logarithm, we obtain
    \begin{equation*}
        \kldiv(\P, \Q) \leq \log \left(\int_{\cX} \frac{d\P}{d\Q}d\P\right) =  \log \left(\int_{\cX} \left(\frac{d\P}{d\Q} \right)^2d\Q\right) =\log \left(\int_{\cX} \left(\frac{d\P}{d\Q} \right)^2d\Q - 1 +1\right) = \log(\chidiv+1) \eqsp.
    \end{equation*}
    Finally, using the inequality $\log (1+x) \leq x$ for any $x > -1$ allows us to conclude.
\end{proof}

\lowerbound*
\begin{proof}[Proof of \Cref{theorem:lower_bound}]
    We first restrict $\theta$ to belong to the set $\Theta = \{\pm \sqrt{d/T}\}^d \subseteq \oball(0,1)$. Let $i \in [d]$ and $\theta, \theta' \in \Theta$ such that for any $j \in [d], j\ne i, \theta_j = \theta_j'$ and $\theta_i' = -\theta_i$. For any prediction $\hat{a}$ output by $\alg$, we define the event
    \begin{equation*}
        A_{i, \theta} = \{\sgn(\hat{a}_i) = -\sgn(\theta_i)\} \eqsp,
    \end{equation*}
    as well as the corresponding probability
    \begin{equation*}
        p(\theta, i) =  \P_{\theta}(A_{i, \theta}) = \P_{\theta}(\{\sgn(\hat{a}_i) = -\sgn(\theta_i)\}) \eqsp.
    \end{equation*}
    Consider the action set $\cA = [\pm 1/\sqrt{d}]^d$. Note that $A_{i, \theta}^\comp =  \{\sgn(\hat{a}_i) = \sgn(\theta_i)\} =  \{\sgn(\hat{a}_i) = -\sgn(\theta_i')\}$. We now apply Bretagnolle-Huber's inequality to obtain
    \begin{equation}\label{fqsgsrgzgaq}
        \P_{\theta}(A_{i, \theta}) + \P_{\theta'}(A_{i, \theta}^\comp) \geq \exp(-\kldiv(\P_{\theta}, \P_{\theta'}))/2 \eqsp.
    \end{equation}
    Now using the expression of the divergence from \eqref{equation_div_expression} as well as \Cref{lemma:inequality_divergences}, we can write
    \begin{align*}
        \kldiv(\P_{\theta}, \P_{\theta'}) \leq \sum_{t=1}^T \E_{\theta}\left[\chidiv(\ber(\sigma(\theta^T B_t)), \ber(\sigma(\theta^{' T} B_t))) \right] \eqsp.
    \end{align*}
    Since $\chidiv(\ber(p), \ber(q)) = (p-q)^2/q^2$, we can write
    \begin{align*}
        \kldiv(\P_{\theta}, \P_{\theta'}) \leq \sum_{t=1}^T \E_{\theta}\left[(\sigma(\theta^T B_t) - \sigma(\theta^{' T}B_t))^2/\sigma(\theta^{' T}B_t)(1-\sigma(\theta^{' T}B_t))\right] \eqsp.
    \end{align*}
    For any $x \in \R^d, \sigma(x) = 1/(1+\rme^{-x})$, which gives that $1/\sigma(x) (1-\sigma(x)) = \rme^{x}(1+\rme^{-x})^2$ and we define $f \colon \R \to \R, x \mapsto \rme^{x}(1+\rme^{-x})^2$. A derivation shows that $f'$ cancels out in $0$, is negative on $\R_-$ and positive on $\R_+$. Therefore, for any $x \in [-1/2, 1/2], f(x) \leq \max\{f(-1/2), f(1/2)\} \leq 5$.
    \newline
    \newline
    We now define $g\colon [0,1] \to \R, v \mapsto \sigma(\theta^{' T} B_t +v(\theta - \theta^{' T})^T B_t)$. We have that $g(1) = \sigma(\theta^T B_t)$ while $g(0) = \sigma(\theta^{' T} B_t)$, which allows us to write
    \begin{align*}
        \sigma(\theta^T B_t) - \sigma(\theta^{' T} B_t) & = g(1) - g(0) \\
        & = \int_{u=0}^1 (\theta - \theta^{'})^T B_t \; \sigma'(\theta^{' T} B_t +v(\theta - \theta^{' T})^T B_t) \rmd v \\
        & = \int_{u=0}^1 \sigma'(\theta^{' T} B_t +v(\theta - \theta^{' T})^T B_t) \rmd v \; \; (\theta - \theta^{'})^T B_t  \eqsp.
    \end{align*}
    As we showed in the proof of \Cref{theorem:simple_regret_upper_bound}, we have that for any $x \in [-2, 2], \sigma'(x) \leq \sigma'(0) = 1/4$ and therefore, we obtain
    \begin{align*}
        \sigma(\theta^T B_t) - \sigma(\theta^{' T} B_t) \leq (\theta - \theta^{'})^T B_t / 4 \eqsp.
    \end{align*}
    Plugging the different inequalities together gives that
    \begin{align*}
    \kldiv(\P_{\theta}, \P_{\theta'}) & \leq \sum_{t=1}^T \E_{\theta}\left[5 ((\theta - \theta^{'})^T B_t)^2/16 \right] \\
    & = 5/16 \sum_{t=1}^T \E_{\theta}\left[((\theta - \theta^{'})^T B_t)^2 \right] \\
    & = 5/16 \sum_{t=1}^T 16 \; \theta_i^2/d \\
    & = 5 \eqsp,
    \end{align*}
    where the last penultimate line holds since $\|B_t\|_{\infty}^2 \leq 4/d$ because of $\cA \subseteq [\pm 1/\sqrt{d}]^d$ and $\theta-\theta' = 2\theta_i e_i$ where $e_i$ stands for the $i$-th basis vector. The last line holds since $\theta_i \in \{\pm \sqrt{d/T}\}$. Finally, plugging this inequality in \eqref{fqsgsrgzgaq} gives that
    \begin{equation}\label{qsfjqsgl}
        p(\theta,i) + p(\theta',i) \geq \rme^{-5}/2 \eqsp.
    \end{equation}
    We now apply the "averaging hammer" technique, which consists in summing all the $p(\theta,i)$ for $\theta \in \Theta, i \in [d]$ and group the term that differ in only one coordinate. It gives that
    \begin{align*}
        \sum_{\theta \in \Theta} 1/|\Theta| \sum_{i=1}^d p(\theta,i) = 1/|\Theta| \sum_{i=1}^d \sum_{\theta \in \Theta} p(\theta,i) \eqsp,
    \end{align*}
    and we reckon that $2^{d-1}$ pairs appear as in \eqref{qsfjqsgl}, which gives that
    \begin{align*}
        \sum_{\theta \in \Theta} 1/|\Theta| \sum_{i=1}^d p(\theta,i) \geq 1/|\Theta| \sum_{i=1}^d 2^{d-1} \rme^{-5}/2 = d \; \rme^{-5}/4 \eqsp,
    \end{align*}
    since $\card(\Theta) = 2^d$ (hypercube). Therefore, there exists at least on $\theta^\star \in \Theta$ such that $\sum_{i=1}^d p(\theta,i) \geq d \rme^{-5}/4$. Still considering the action set $\cA = [-1/\sqrt{d}, 1/\sqrt{d}]^d \subseteq \oball(0,1)$, we can lower bound the regret
    \begin{align*}
        \regret(T, (\cA_n)_{n \in [N]}, \theta^\star) & \geq \E_{\theta^\star}\left[\sum_{i=1}^d (\sgn(\theta_i^\star)/\sqrt{d}-\hat{a}_i)\theta_i^\star \right] \\
        & \geq \sum_{i=1}^d \P_{\theta^\star}(\sgn(\theta_i^\star) \ne \sgn(\hat{a}_i))|\theta_i^\star|/\sqrt{d} \\
        & = \frac{1}{\sqrt{T}} \sum_{i=1}^d \P_{\theta^\star}(\sgn(\theta_i^\star) \ne \sgn(\hat{a}_i)) \\
        & \geq d \, \rme^{-5}/(4 \sqrt{T}) \eqsp.
    \end{align*}
\end{proof}

\section{Supplementary theorems and algorithms}
\label{appendix:algorithms}

\begin{theorem}[Kiefer-Wolfowitz]\label{theorem:kiefer_wolfowitz}
    Assume that the action set $\cB$ is such that Span$(\cB) = \R^d$. Since $\cB \subseteq \oball(0,2)$ and $\cB$ is finite, $\cB$ is compact. Therefore, the following are equivalent
    \begin{itemize}
        \item $\pi^\star$ is a minimizer of $g$,
        \item $\pi^\star$ is a maximizer of $\pi \mapsto \log \det \matopt(\pi)$,
        \item $g(\pi^\star) = d$,
    \end{itemize}
    where the quantities $g$ and $\matopt$ are defined in \eqref{equation:definition_V_and_g}. Furthermore, there exists such a $\pi^\star$ with a support of size smaller than $d(d+1)/2$.
\end{theorem}

\begin{algorithm}[!ht]
\caption{$\algfrankwolfe$: \texttt{Frank-Wolfe Algorithm}}\label{algorithm:frank_wolfe_algorithm}
\begin{algorithmic}[1]
    \State {\bfseries Input:} Set of actions $\cB = \{b_n\}_{l \in [L]}$, initial distribution $\pi_0$ over this set of actions, precision $\epsilon$, regularization parameter $\lambda$.
    \State {\bfseries Compute} $\matopt(\pi_0) = \lambda I, m=0$.
    \While{$g(\hat{\pi}_m) > \sqrt{(1+\epsilon)}d$}
        \State Compute $b_m = \argmax_{b \in \cB} \|b\|_{\matopt(\pi_m)^{-1}}$.
        \State $\gamma_m = \argmax_{\gamma \in [0,1]} \log \det(V((1-\gamma) \pi_m + \gamma \1_{b_m}))$
        \State For any $b \in \cB, \pi_{m+1} (b) = (1-\gamma_m)\pi_m(b) + \gamma_m \1_{b_k}(b)$.
        \State Update $\matopt(\pi_{m+1})$.
    \EndWhile
    \State Output the estimated policy $\hat{\pi} = \pi_m$.
\end{algorithmic}
\end{algorithm}

\end{document}